\pdfoutput=1

\documentclass[11pt]{article}

\usepackage{acl}

\usepackage{times}
\usepackage{latexsym}

\usepackage{dsfont}
\newcommand{\citeposs}[1]{\citeauthor{#1}'s (\citeyear{#1})}

\usepackage{xfrac}
\usepackage{cancel}

\usepackage{emoji}

\usepackage{mathtools}
\allowdisplaybreaks

\usepackage[T1]{fontenc}

\usepackage{microtype}
\everypar{\looseness=-1}
\usepackage{layouts}
\usepackage{adjustbox}
\usepackage{color}
\usepackage{colortbl}
\usepackage{multirow}
\usepackage{bbm}
\usepackage{amsthm}
\usepackage{amsfonts}
\usepackage{amsmath}
\usepackage{amssymb}
\usepackage{booktabs}
\usepackage{enumitem}

\usepackage[belowskip=-10pt,aboveskip=10pt]{caption}

\newtheorem{lemma}{Lemma}
\newtheorem{proposition}{Proposition}
\newtheorem{corollary}{Corollary}
\newtheorem{definition}{Definition}
\newtheorem{fact}{Fact}
\newtheorem{estimator}{Estimator}

\usepackage{cleveref}
\crefname{section}{\S}{\S\S}
\Crefname{section}{\S}{\S\S}
\crefname{table}{Table}{}
\crefname{figure}{Figure}{}
\crefname{algorithm}{Algorithm}{}
\crefname{equation}{Eq.}{}
\crefname{appendix}{App.}{}
\crefname{lemma}{Lemma}{}
\crefname{definition}{Definition}{}
\crefname{proposition}{Proposition}{}
\crefformat{section}{\S#2#1#3}

\newcommand{\defeq}[0]{\mathrel{\stackrel{\textnormal{\tiny def}}{=}}}
\usepackage{relsize}

\newcommand{\calX}{\mathcal{X}}
\newcommand{\data}{\mathcal{D}}
\newcommand{\ent}{\mathrm{H}}
\newcommand{\enthat}[1]{\widehat{\ent}_{\textsc{#1}}}
\newcommand{\enthatmle}{\widehat{\ent}_{\textsc{mle}}(\calD)}
\newcommand{\phat}{\widehat{p}_{\textsc{mle}}}
\newcommand{\pnsb}{p_{\scriptsize \textsc{nsb}}}
\newcommand{\KL}{\mathrm{KL}}
\newcommand{\expect}[1]{\mathbb{E}\left[#1\right]}
\newcommand{\expectp}[1]{\mathbb{E}_p\left[#1\right]}
\newcommand{\bias}[1]{\mathrm{bias}\left(#1\right)}

\newcommand{\Dirichlet}{\mathrm{Dirichlet}}

\newcommand{\ddp}{\mathrm{d}p}
\newcommand{\calD}{\mathcal{D}}
\newcommand{\valpha}{\boldsymbol \alpha}

\newcommand{\defn}[1]{\textbf{#1}}
\newcommand{\defndata}{Let $\mathcal{D}$ be our dataset of size $N$ sampled from $p$.}

\newcommand{\phatxk}{\phat(x_k)}
\newcommand{\pxk}{p(x_k)}
\newcommand{\cxk}{c(x_k)}

\newcommand{\xk}{x_k}

\definecolor{green0}{HTML}{e9f5ea}
\definecolor{green1}{HTML}{d3ecd4}
\definecolor{green2}{HTML}{bde2bf}
\definecolor{green3}{HTML}{a8d8aa}
\definecolor{green4}{HTML}{92ce95}
\definecolor{green5}{HTML}{7cc57f}

\usepackage[disable]{todonotes}
\makeatletter
\newcommand*\iftodonotes{\if@todonotes@disabled\expandafter\@secondoftwo\else\expandafter\@firstoftwo\fi}  
\makeatother

\title{Estimating the Entropy of Linguistic Distributions}

\newcommand{\ethz}{\emoji[openmoji]{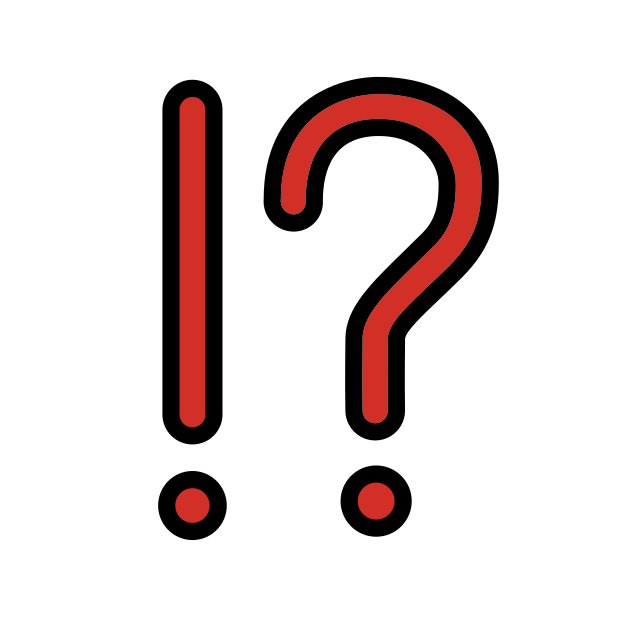}}
\newcommand{\georgetown}{\emoji[openmoji]{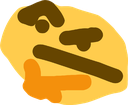}}

\newcommand{\emo}[1]{\raise1.0ex\hbox{\normalfont#1}}

\renewcommand{\digamma}{\psi}
\newcommand{\trigamma}{\psi_1}

\usepackage{inconsolata}

\title{Estimating the Entropy of Linguistic Distributions}

\author{
Aryaman Arora\emo{\georgetown}~\;~Clara Meister\emo{\ethz}~\;~Ryan Cotterell\emo{\ethz}\\
  \emo{\georgetown}Georgetown University~\;~\emo{\ethz}ETH Z{\"u}rich \\
  \texttt{\href{mailto:aa2190@georgetown.edu
}{aa2190@georgetown.edu}}~\;~\texttt{\{\href{mailto:clara.meister@inf.ethz.ch}{clara.meister},\href{mailto:ryan.cotterell@inf.ethz.ch}{ryan.cotterell}\}@inf.ethz.ch}
}

\begin{document}
\maketitle
\begin{abstract}
Shannon entropy is often a quantity of interest to linguists studying the communicative capacity of human language.   
However, entropy must typically be estimated from observed data because researchers do not have access to the underlying probability distribution that gives rise to these data.
While entropy estimation is a well-studied problem in other fields, there is not yet a comprehensive exploration of the efficacy of entropy estimators for use with \textit{linguistic} data.
In this work, we fill this void, studying the empirical effectiveness of different entropy estimators for linguistic distributions. 
In a replication of two recent information-theoretic linguistic studies, we find evidence that the reported effect size is over-estimated due to over-reliance on poor entropy estimators. 
Finally, we end our paper with concrete recommendations for entropy estimation depending on distribution type and data availability.
\end{abstract}

\section{Introduction}
There is a natural connection between information theory, the mathematical study of communication systems, and linguistics, the study of human language---the primary vehicle that humans employ to communicate. 
Researchers have exploited this connection since information theory's inception \cite{shannon1951prediction,cherry,harris1991theory}.
With the advent of modern computing, the number of information-theoretic linguistic studies has risen, exploring claims about language such as the optimality of the lexicon \citep{piantadosi2011word,pimentel-etal-2021-non}, the complexity of morphological systems \citep{cotterell-etal-2019-complexity,wu-etal-2019-morphological,rathi-etal-2021-information}, and the correlation between surprisal and language processing time \citep[\textit{inter alia}]{smith2013-log-reading-time,bentz2017entropy,goodkind-bicknell-2018-predictive,cotterell-etal-2018-languages,meister-etal-2021-revisiting}.

\begin{figure}
    \centering
    \includegraphics[width=\linewidth]{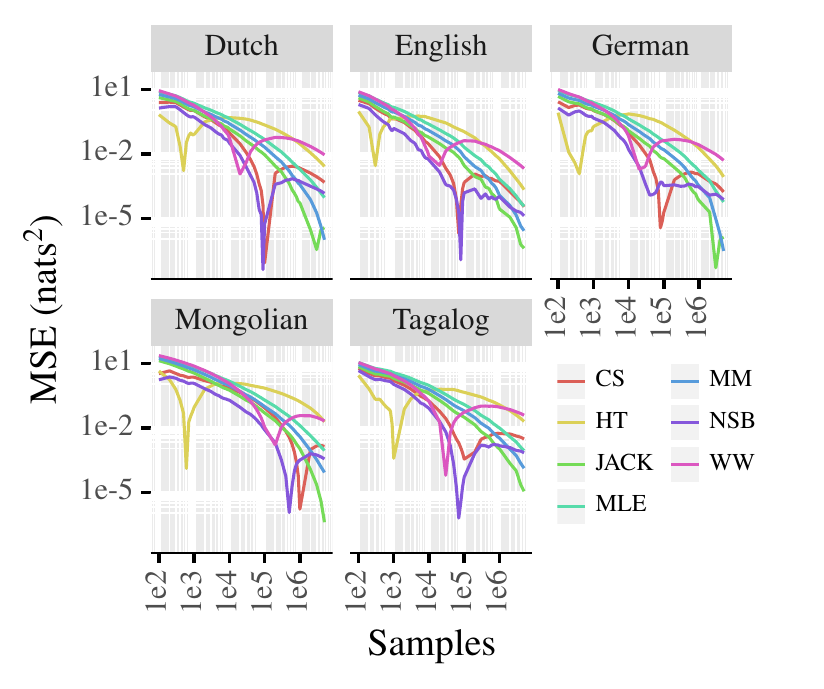}
    \caption{A comparison of several estimators of the entropy of the unigram distribution across 5 languages. Minima in all the graphs indicate sign changes in the error of the estimate, from an under- to an over-estimate.}
    \label{fig:celex_big}
\end{figure}

In information-theoretic linguistics, a fundamental quantity of research interest is entropy.
Entropy is both useful to linguists in its own right, and is necessary for estimating other useful quantities, e.g., mutual information.
However, the estimation of entropy from raw data can be quite challenging \cite{paninski,nowozin-1}, e.g., in expectation, the plug-in estimator \emph{underestimates} entropy \citep{miller1955note}.
Linguistic distributions often present additional challenges.
For instance, 
many linguistic distributions, such as the unigram distribution, follow a power law \cite{zipf,mitzenmacher2004brief}.\footnote{As \citet{nsb} highlight, when estimating the entropy of a distribution that follows a power law, it is often possible to get an effectively meaningless estimate that is completely determined by the estimator's hyperparameters.}  
Linguistics is not the only field with such nuances, and so a large number of entropy estimators have been proposed in other fields \citep[\textit{inter alia}]{chao2003nonparametric,JMLR:v15:archer14a}.
However, no work to date has attempted a practical comparison of these estimators on \emph{natural language} data.
This work fills this empirical void.

Our paper offers a large empirical comparison of the performance of 6 different entropy estimators on both synthetic and natural language data, an example of which is shown in \cref{fig:celex_big}.
We find that \citeposs{chao2003nonparametric} is the best estimator when very few data are available, but \citeposs{nsb} is superior as more data become available.
Both are significantly better (in terms of mean-squared error) than the na{\"i}ve plug-in estimator. 
Importantly, we also show that two recent studies \citep{williams+al.tacl21,mccarthy-etal-2020-measuring} show smaller effect sizes when a better estimator is employed; however, we are able to reproduce a significant effect in both replications.
We recommend that future studies carefully consider their choice of entropy estimators, taking into account data availability and the nature of the underlying distribution.\footnote{Our code is available at \url{https://github.com/aryamanarora/entropy-estimation}.}\looseness=-1

\section{Entropy and Language}\label{sec:entlang}

Shannon entropy is a quantification of the uncertainty in a random variable. Given a (discrete) random variable $X$ with probability distribution $p$ over $K$ possible outcomes $\calX = \{\xk\}_{k=1}^K$, the Shannon entropy of $X$ is defined as
\begin{equation}\label{eq:ent}
   \ent(X) =  \ent(p) \defeq -\sum_{k=1}^K \pxk \log \pxk
\end{equation}
Entropy has many uses throughout science and engineering; for instance, 
\citet{shannon} originally proposed entropy as a lower bound on the compressibility of a stochastic source.

Yet the application of information-theoretic techniques to linguistics is not so straightforward: Information-theoretic measures are defined over probability distributions and, in the study of natural language, we typically only have access to \emph{samples} from the distribution of interest, e.g., the phonotactic distribution in English, which permits word we cannot find in a corpus, like \textit{blick}, rather than the true probabilities required in the computation of \cref{eq:ent}. Indeed, it is often the case that not all elements of $\calX$ are even observed in available data---such as words that were coined after the a corpus was collected.

Rather, $p$ must be approximated in order to estimate $\ent(p)$. 
One solution is \defn{plug-in estimation}: Given samples from $p$, the maximum-likelihood estimate for $p$ is ``plugged'' into \cref{eq:ent}.
However, as originally noted by \citet{miller1955note}, this strategy generally yields poor estimates.\footnote{A proof of this result in given in full in \cref{prop:negative-bias}.}  
It is thus necessary to derive more nuanced estimators.

\section{Statistical Estimation Theory}\label{sec:theory}

Statistical estimation theory provides us with the tools for estimating various quantities of interest based on samples from a distribution.

Central to this theory is the \defn{estimator}: A statistic that approximates a property of the distribution our data is drawn from.
More formally, let $\data = \{\widetilde{x}^{(n)}\}_{n=1}^N$ be samples from an unknown distribution $p$.
Suppose we are interested in a quantity $\theta$ that can be computed as a function of the distribution $p$. 
An estimator $\widehat\theta(\data)$ for $\theta$ is then a function of the data $\data$ that provides an approximation of $\theta$.\looseness=-1

Two properties of an estimator are often of interest: \defn{bias}---the difference between the true value of $\theta$ and the expected value of our estimator $\widehat\theta(\data)$ under $p$---and \defn{variance}---how much $\widehat\theta(\data)$ fluctuates from sample set to sample set: \looseness=-1
\begin{align}
    \mathrm{bias}(\widehat\theta(\data)) &\defeq \mathbb{E}_{p}[\widehat\theta(\data)] - \theta \label{eq:bias}\\
    \mathrm{var}(\widehat{\theta}(\data)) &\defeq \mathbb{E}_p[(\widehat{\theta}(\data) - \mathbb{E}_p[\widehat{\theta}(\data)]) ^2]
\end{align}
It is desirable to construct an estimator that has both low bias and low variance. 
However, the \defn{bias--variance} trade-off tells us that we often have to pick one, and we should focus on a balance between the two.
This trade-off is evinced through mean-squared error (MSE), a metric oft-employed for assessing estimator quality:
\begin{equation}
    \mathrm{MSE}(\widehat\theta(\data)) = \mathrm{bias}(\widehat{\theta}(\data))^2 + \mathrm{var}(\widehat{\theta}(\data))\label{eq:mse}
\end{equation}
To recognize the trade-oft note that, for any fixed MSE, a decrease in bias must be compensated with an increase in variance and vice versa.
Indeed, it is important to recognize that there is typically  no single estimator that is seen as ``best.'' Different estimators balance the bias--variance trade-off differently, making their perceived quality specific to one's use-case.  
Importantly, the effectiveness of an estimator also depends on the domain of interest. Consequently, an empirical study of various entropy estimators, which this paper provides, is necessary in order to determine which entropy estimators are best suited for linguistic distributions.

\subsection{Plug-in Estimation of Entropy}
A simple, two-step approach for estimating entropy is \defn{plug-in} estimation.
In the first step, we compute the maximum-likelihood estimate for $p$ from our dataset $\data$ as follows
\begin{equation}\label{eq:mle-estimate}
\phatxk \defeq \frac{\sum_{n=1}^N \mathbbm{1}\{\widetilde{x}^{(n)} = \xk\}}{N}
\end{equation}
In the second step, we plug \cref{eq:mle-estimate} into \cref{eq:ent} directly, which results in the estimator $\enthat{mle}(\calD)$.
So why is this a bad idea? While our probability estimates themselves are unbiased, entropy is a concave function. 
Consequently, by Jensen's inequality, this estimator is, in expectation, a \emph{lower bound} on the true entropy (see \cref{app:mle_bias} for proof). 
Moreover, when $N \ll K$, which is often the case in power-law distributed data, the estimate becomes quite unreliable \cite{nsb}.

\subsection{An Ensemble of Entropy Estimators}\label{sec:ent}

\paragraph{MM---\citet{miller1955note} and \citet{madow1948limiting}.}
The first innovation in entropy estimation known to the authors is a simple fix derived from a first-order Taylor expansion of MLE (described above).
The Miller--Madow estimator only involves a simple additive correction, which is shown below:
\begin{equation}
 \enthat{mm}(\calD) \defeq  \enthat{mle}(\calD) + \frac{K - 1}{2N}
\end{equation}
where $K$ is size of the support of $\calX$.
The Miller--Madow correction should seem intuitive in that we add $\frac{K-1}{2N} \geq 0$ to compensate for the negative bias of the estimator.
A full derivation of the Miller--Madow estimator is given in \cref{prop:miller-madow}.

\paragraph{JACK---\citet{zahl1977}.} 
Next we consider the jackknife, which is a common strategy used to correct for the bias of statistical estimators.
In the case of entropy estimation, we can apply the jackknife out of the box to correct the bias inherent in the MLE estimator.
Explicitly, this is done by averaging plug-in entropy estimates $\enthat{mle}(\calD)$ albeit with the $n$\textsuperscript{th} sample from the data removed; we denote this held-out plug-in estimator as $\enthat{mle}^{\setminus n}(\calD)$.
Averaging these ``held-out'' plug-in estimators results in the following simple entropy estimator
\begin{equation}
    \enthat{jack}(\calD) \defeq N\,\enthat{mle}(\calD) - \frac{N-1}{N}\sum_{n=1}^N{ \enthat{mle}^{\setminus n}(\calD)}
\end{equation}
Note that the jackknife is applicable to any estimator, not just $\enthat{mle}(\calD)$, and, thus, can be combined with any of the other approaches mentioned. 

\paragraph{HT---\citet{horvitz-thompson}.}
Horvitz--Thompson is a general scheme for building estimators that employs importance weighting in order to more efficiently estimate a function of a random variable. Importantly, this estimator gives us the ability to compensate for situations where  the probability of an outcome is so low that it is often not observed in a sample, which is often the case for e.g., power-law distributions.  

While a full exposition of HT estimators is outside of the scope of this work, in essence, we can divide the expected probability of a class by each class's estimated inclusion probability to compensate for such situations. Given the true probability of an outcome $\pxk$, the probability that it occurs at least once in a sample of size $N$ is $1 - (1 - \pxk)^N$.
The HT estimator for entropy is then defined as
\begin{equation}\label{eq:ht}
\enthat{ht}(\calD) \defeq -\sum_{k=1}^K \frac{\phatxk \log{\phatxk}}{1 - (1 - \phatxk)^N}
\end{equation}
using our MLE probability estimates $\phatxk$.

\begin{table*}[t]
    \centering
    \adjustbox{max width=\linewidth}{
    \small
    \begin{tabular}{lrrrr|rrrr}
        \toprule
        & \multicolumn{4}{c}{\textbf{MAB}} & \multicolumn{4}{c}{\textbf{MSE}} \\
        & $10^2$ & $10^3$ & $10^4$ & $10^5$ & $10^2$ & $10^3$ & $10^4$ & $10^5$ \\
        \midrule
        English & \cellcolor{green5}HT & \cellcolor{green5}HT & \cellcolor{green4}NSB & \cellcolor{green5}NSB & \cellcolor{green5}HT & \cellcolor{green5}HT & \cellcolor{green4}NSB & \cellcolor{green3}NSB \\
        German & \cellcolor{green5}HT & \cellcolor{green5}HT & \cellcolor{green5}NSB & \cellcolor{green5}CS & \cellcolor{green5}HT & \cellcolor{green5}HT & \cellcolor{green5}NSB & \cellcolor{green4}CS \\
        Dutch & \cellcolor{green5}HT & \cellcolor{green5}HT & \cellcolor{green5}NSB & \cellcolor{green4}CS  & \cellcolor{green5}HT & \cellcolor{green5}HT & \cellcolor{green4}NSB & \cellcolor{green4}CS  \\
        \midrule
        Mongolian & \cellcolor{green5}NSB & \cellcolor{green5}HT & \cellcolor{green5}NSB & \cellcolor{green5}NSB & \cellcolor{green4}NSB & \cellcolor{green5}HT & \cellcolor{green5}NSB & \cellcolor{green5}NSB \\
        Tagalog & \cellcolor{green5}HT & \cellcolor{green5}HT & \cellcolor{green5}NSB & \cellcolor{green5}NSB & \cellcolor{green5}HT & \cellcolor{green5}HT & \cellcolor{green5}NSB & \cellcolor{green4}NSB \\
        \bottomrule
    \end{tabular}}
    \caption{The best unigram entropy estimators on the corpora studied, tested on various $N$ averaged over $100$ samples. All differences are statistically significant on the permutation test; lighter color indicates fewer statistically significant comparisons on the Tukey test. \textit{Scale}: significantly better than \colorbox{green5}{6}\colorbox{green4}{5}\colorbox{green3}{4}\colorbox{green2}{3}\colorbox{green1}{2}\colorbox{green0}{1}\colorbox{white}{0} other estimators.}
    \label{tab:unigram}
\end{table*}

\paragraph{CS---\citet{chao2003nonparametric}.}
Chao--Shen modifies HT by multiplying the MLE probability estimates by an estimate of sample coverage.
Formally, let $f_1$ be the number of observed singletons\footnote{A singleton (\textit{hapax legomenon}) is an outcome which is observed only once in the sample.} in sample; our sample coverage can be estimated as $\widehat{C} = 1 - \frac{f_1}{N}$. The CS estimator is then computed as:\looseness=-1
\begin{equation}
   \enthat{cs}(\calD) \defeq -\sum_{k=1}^K \frac{\widehat{C} \cdot \phatxk \log{\widehat{C}\cdot\phatxk}}{1 - (1 - \widehat{C} \cdot\phatxk)^N}
\end{equation}
In the case that $f_1 = N$, we set $f_1 = N - 1$ to ensure the estimated entropy is not $0$.

\paragraph{WW---\citet{wolpert1995estimating}.} 
One family of entropy estimators in information theory is based on Bayesian principles. The first of these was the Wolpert--Wolf estimator, which uses a Dirichlet prior (with concentration parameter $\alpha$ and a uniform base distribution).
This Bayesian estimator has a clean, closed form:
\begin{equation}\label{eq:ww}
    \enthat{ww}(\calD \mid \valpha) \defeq \digamma\left(\widetilde{A} + 1\right)
    - \sum_{k=1}^K\frac{\widetilde{\alpha}_k}{\widetilde{A}} \digamma(\widetilde{\alpha}_k + 1) 
\end{equation}
where $\widetilde{\alpha}_k = \cxk + \alpha_k$ (for  the histogram count $\cxk$ of class $k$ in the sample; this is analogous to Laplace smoothing), $\widetilde{A} = \sum_{k=1}^K \widetilde{\alpha}_k$, and $\digamma$ is the digamma function.
A full derivation of \cref{eq:ww} is given in \cref{prop:ww}.
Unfortunately, \cref{eq:ww} is very dependent on the choice of $\valpha$: For large $K$, $\valpha$ almost completely determines the final entropy estimate, an observation first made by \citet{nsb} which motivated their improved estimator described below.

\paragraph{NSB---\citet{nsb}.}
\citeauthor{nsb} (NSB) attempt to alleviate the Wolpert--Wolf estimator's dependence on $\valpha$.
They take $\valpha = \alpha \cdot \boldsymbol{1}$, enforcing that the Dirichlet prior is symmetric, and develop a hyperprior over $\alpha$ that results in a near-uniform distribution over entropy.
The hyperprior is given by
\begin{equation}\label{eq:nsb-main}
\pnsb(\alpha) \defeq \frac{K\trigamma(K\alpha + 1) - \trigamma(\alpha + 1)}{\log{K}}
\end{equation}
where $\trigamma$ is the trigamma function.
 A full derivation of \cref{eq:nsb-main} is given in \cref{prop:nsb}.
This choice of hyperprior mitigates the effect that the chosen $\alpha$ has on the entropy estimate.
\citeposs{nsb} entropy estimator is then the posterior mean of the Wolpert--Wolft  estimator taken under $\pnsb$:
\begin{equation}\label{eq:nsb-final}
    \enthat{nsb}(\calD) = \int_0^\infty \enthat{ww}(\calD \mid \alpha\cdot \boldsymbol{1})\,\pnsb(\alpha)\, \mathrm{d}\alpha
\end{equation}
Typically, numerical integration is used to quickly compute the unidimensional integral.

\section{Experiments}
Here we provide an evaluation of the entropy estimators presented in \cref{sec:ent} on linguistic data. \looseness=-1

\begin{table*}[t]
    \centering
    \adjustbox{max width=\linewidth}{
    \small
    \begin{tabular}{lr|rrrrrr}
        \toprule
        \textbf{Language} & $n$ & \textbf{MLE} & \textbf{CS} & \textbf{MM} & \textbf{JACK} & \textbf{WW} & \textbf{NSB} \\
        \midrule
        Italian & $16,856$ & \cellcolor{green5}$20.00\%$ & \cellcolor{green5}$15.56\%$ & \cellcolor{green5}$16.43\%$ & \cellcolor{green5}$14.09\%$ & \cellcolor{green5}$19.67\%$ & \cellcolor{green5}$11.41\%$ \\
        Polish & $15,525$ & $30.52\%$ & $23.48\%$ & $25.49\%$ & \cellcolor{green5}$21.75\%$ & $34.68\%$ & \cellcolor{green5}$17.07\%$ \\
        Portuguese & $7,409$ & \cellcolor{green5}$27.60\%$ & \cellcolor{green5}$20.76\%$ & \cellcolor{green5}$22.51\%$ & \cellcolor{green5}$18.81\%$ & \cellcolor{green5}$33.32\%$ & \cellcolor{green5}$14.18\%$ \\
        Spanish & $21,408$ & \cellcolor{green5}$20.50\%$ & \cellcolor{green5}$15.17\%$ & \cellcolor{green5}$16.44\%$ & \cellcolor{green5}$13.80\%$ & \cellcolor{green5}$21.04\%$ & \cellcolor{green5}$10.50\%$ \\
        \midrule
        Arabic & $2,483$ & $45.31\%$ & \cellcolor{green5}$38.49\%$ & $40.99\%$ & \cellcolor{green5}$37.93\%$ & \cellcolor{green5}$49.09\%$ & \cellcolor{green5}$34.82\%$ \\
        Croatian & $13,856$ & $31.35\%$ & $26.04\%$ & $26.62\%$ & $23.08\%$ & $35.66\%$ & $19.06\%$\\
        Greek & $3,305$ & $41.58\%$ & $33.17\%$ & $36.39\%$ & $32.32\%$ & $48.80\%$ & $27.00\%$ \\
        \bottomrule
    \end{tabular}}
    \caption{Normalized mutual information, calculated with several estimators, between adjectives and the inanimate nouns they modify based on UD corpora. Colored-in cell means statistically significant NMI value.}
    \label{tab:williams}
\end{table*}

\subsection{Entropy of the Unigram Distribution}\label{sec:unigram}
We start our study with a controlled experiment where we estimate the entropy of the truncated unigram distribution, the (finite) distribution over the frequent word tokens in a language without regard to context \citep{baayen2016frequency,diessel2017usage,divjak2019frequency,nikkarinen-etal-2021-modeling}.
We renormalize the frequency counts of corpora in English, German, and Dutch \citep[taken from CELEX;][]{celex}, as well as Mongolian and Tagalog (from Wikipedia\footnote{We used dumps from November 1, 2021: \href{https://dumps.wikimedia.org/tlwiki/20211101/}{Mongolian} and \href{https://dumps.wikimedia.org/mnwiki/20211101/}{Tagalog}; the extracted counts are available in our repository.}).
We take this renormalization as a gold standard distribution, since we cannot access the underlying unigram distribution.
We then draw samples of varying sizes ($N \in \{10^2, 10^3, 10^4, 10^5\}$) from the distribution of renormalized frequency counts to test the estimators' ability to recover the underlying distributions' entropy.
While the renormalized frequency counts are not necessarily representative of the \emph{true} unigram distribution, they nevertheless provide us with a controlled setting to benchmark various entropy estimators.

We evaluate the estimators on both bias and MSE, as defined in \cref{eq:bias,eq:mse}, as well as mean absolute bias (MAB).
To test the statistical significance of differences in metrics between entropy estimators, we use paired permutation tests \cite{perms} (sampling $1,000$ permutations) between pairs of estimators, checking MAB and MSE.
We run \citeauthor{tukey}'s test (\citeyear{tukey}) to judge the statistical significance of differences in MAB and MSE between all pairs of estimators, which found only a few insignificant comparisons when $N$ was large.

Results are shown in \cref{tab:unigram} and \cref{fig:celex_big}. We find that NSB (followed closely by CS) converges almost to the true entropy from below using with only a few samples. HT is the best estimator for $N < 2,000$, but as $N$ increases it tends to overestimate entropy to the point where its bias is greater than that of MLE. 
Besides HT, all estimators at all tested sample sizes $N$ have lower MAB and MSE than MLE.

\subsection{Replication of \citet{williams+al.tacl21} }
Next, we turn to a replication of \citeposs{williams+al.tacl21} information-theoretic study on the association between gendered inanimate nouns and their modifying adjectives.
They estimate mutual information by using its familiar decomposition as the difference of two entropies: $\mathrm{MI}(X; Y) = \ent(X) - \ent(X \mid Y)$.
The entropies $\ent(X)$ and $\ent(X \mid Y)$ are estimated independently and then their difference is computed.
We replicate \citeposs{williams+al.tacl21} experiments using gold-parsed Universal Dependencies corpora, filtering out animate nouns with Multilingual WordNet \citep{bond-foster-2013}.
We rerun their experimental set-up using our full suite of entropy estimators to determine whether the relationship they posit remains significant, checking 3 more languages not in the original study.\looseness=-1

We report results for normalized mutual information (dividing MI by maximum possible MI) in \cref{tab:williams}.
We find that using NSB (the estimator we found most effective in \cref{sec:unigram}) instead of MLE, nearly halves the measured effect in all languages. 
However, the effect remains statistically significant in 5 of 7 languages tested, including the 4 that were also in the original study.

\subsection{Replication of \citet{mccarthy-etal-2020-measuring}}
Finally, we turn our attention to \citeposs{mccarthy-etal-2020-measuring} study on the similarity between grammatical gender partitions between languages. Using information-theoretic measures, they found that closely related languages have more similar gender groupings of core lexical items. We replicate their experiment on Swadesh lists \citep{swadesh1955towards} for 10 European languages with different estimators, and find that hierarchical clustering over both mutual (MI) and variational information (VI) produces the same trees as the original study.
In this case, using NSB, our recommended estimator, results in a reduced estimate of MI (e.g.~Croatian--Slovak: $0.54$ with MLE $\to$ $0.46$ with NSB), but significance testing with 1,000 permutations finds the same pairs were statistically significant for both MI and VI regardless of estimator:~all pairs of Slavic languages and Romance languages, and Bulgarian--Spanish (see \cref{fig:mccarthy}).
Thus, we see a similar result here as in the previous replication.
\section{Conclusion}
This work presents the first empirical study comparing the performance of various entropy estimators for use with natural language distributions.
From experiments on synthetic data (appendix) and natural data (CELEX), and two replication studies of recent papers in information-theoretic linguistics, 
we find that the oft-employed plug-in estimator of entropy can cause misleading results, e.g., the over-estimates of effect sizes seen in both replication studies. 
The recommendation of our paper is that researchers should carefully consider  their choice of entropy estimator based on data availability and the nature of the underlying distribution.


\section*{Ethics Statement}
The authors foresee no ethical concerns with the research presented in this paper.

\section*{Acknowledgments}
We thank Adina Williams, Lucas Torroba Hennigen, Tiago Pimentel, and the anonymous reviewers for feedback on the manuscript.

\bibliography{anthology,custom}
\bibliographystyle{acl_natbib}

\onecolumn
\appendix

\section{Implementation}

The code for each of the entropy estimators is implemented in Python using \texttt{numpy} \citep{harris2020array}, except for NSB which was taken from an existing efficient implementation in the \texttt{ndd} module \citep{marsili}. We calculated entropies with base $e$ (in nats).


\section{Experiments with simulated data}\label{app:sim_data}
In our experiments with simulated data, we explore distributions sampled from a symmetric Dirichlet prior with varying number of classes $K$ and known distributions of Zipfian form with various parameters. Words in natural languages have a roughly Zipfian distribution, with probability inversely proportional to rank \citep{zipf}, and a symmetric Dirichlet distribution is analogous to e.g.~POS tag label distributions in natural language. Thus, studying synthetic data from such distributions as a start is useful.

\subsection{Experiment 1: Symmetric Dirichlet distributions}

We sample $1,000$ distributions from a symmetric Dirichlet distribution with variable number of classes $K$, i.e.~with paramater $\alpha = [\alpha_1, \ldots, \alpha_K] = [1, \ldots, 1]$.
We calculate entropy estimates on different sample sizes $N$. Since we know the parameters of the true distribution, we can compare estimates with the true entropy. We do pairwise comparisons of the MAB and MSE of estimators, using paired permutation tests to establish significance.
\Cref{tab:symmetric_tables} shows our results, including significance tests. It is clear that when $N \gg K$, all of the estimators have nearly converged to the true value and estimator choice does not matter. However, in the low-sample regime some estimators are indeed significantly better at approximating the true entropy. Our results are mixed as to which estimator is best in what context; the one found to be most frequently significantly better than other estimators was Chao--Shen. 
What is clear is that MLE is never the best choice.

\begin{table}[]
    \centering
    \small
    \begin{tabular}{lrrrr|rrrr}
        \toprule
        & \multicolumn{4}{c}{\textbf{MAB}} & \multicolumn{4}{c}{\textbf{MSE}} \\
        & $10^1$ & $10^2$ & $10^3$ & $10^4$ & $10^1$ & $10^2$ & $10^3$ & $10^4$ \\
        \midrule
        $2$ & \cellcolor{green3}HT & \cellcolor{green0}WW & \cellcolor{green0}WW & \cellcolor{green0}WW & \cellcolor{green5}WW & \cellcolor{green2}WW & \cellcolor{green2}WW & \cellcolor{green0}JACK \\
        $5$ & \cellcolor{green5}MM & \cellcolor{green4}WW & WW & JACK & \cellcolor{green5}MM & \cellcolor{green5}WW & \cellcolor{green0}WW & MM \\
        $10$ & \cellcolor{green5}JACK & \cellcolor{green1}CS & \cellcolor{green2}WW & MM & \cellcolor{green5}JACK & \cellcolor{green1}WW & WW & MLE \\
        $100$ & \cellcolor{green5}CS & \cellcolor{green4}CS & \cellcolor{green4}JACK & \cellcolor{green3}WW & \cellcolor{green5}CS & \cellcolor{green4}JACK & \cellcolor{green4}JACK & \cellcolor{green3}WW \\
        $1000$ & \cellcolor{green5}CS & \cellcolor{green5}HT & \cellcolor{green5}CS & \cellcolor{green5}JACK & \cellcolor{green5}CS & \cellcolor{green5}HT & \cellcolor{green5}CS & \cellcolor{green5}JACK \\
        \bottomrule
    \end{tabular}
    \caption{Estimators with least MAB (mean absolute bias) and MSE (mean squared error) for various combinations of $N$ and $K$ sampling from \textbf{symmetric Dirichlet}. The lighter the color the fewer estimators the best estimator was found to be statistically significantly better than.}
    \label{tab:symmetric_tables}
\end{table}

\subsection{Experiment 2: Zipfian distributions}

\begin{table}[]
    \centering
    \small
    \begin{tabular}{lrrrr|rrrr}
        \toprule
        & \multicolumn{4}{c}{\textbf{MAB}} & \multicolumn{4}{c}{\textbf{MSE}} \\
        & $10^1$ & $10^2$ & $10^3$ & $10^4$ & $10^1$ & $10^2$ & $10^3$ & $10^4$ \\
        \midrule
        $100$ & \cellcolor{green5}CS & \cellcolor{green5}CS & \cellcolor{green4}CS & \cellcolor{green3}J & \cellcolor{green5}CS & \cellcolor{green5}CS & \cellcolor{green4}CS & \cellcolor{green3}J\\
        $1000$ & \cellcolor{green4}NSB & \cellcolor{green5}HT & \cellcolor{green4}NSB & \cellcolor{green4}J & \cellcolor{green5}CS & \cellcolor{green5}HT & \cellcolor{green4}NSB & \cellcolor{green5}J\\
        \bottomrule
    \end{tabular}
    \caption{Estimators with least MAB (mean absolute bias) and MSE (mean squared error) for various combinations of $N$ and $K$ sampling from \textbf{Zipfian distributions}.}
    \label{tab:zipf_tables}
\end{table}

We sample $1,000$ finite Zipfian distributions with $K$ classes which obey Zipf's law, that the probability of an outcome is inverse proportional to its rank. The experimental setup is the same as in Experiment 1. A Zipfian distribution approximates (but is not a perfect model of) the distribution of tokens in natural language text in some languages, including English, which was the basis for the law being proposed. Compare similar experiments on infinite Zipf distributions by \citet{zhang2012entropy}. Results are in \cref{tab:zipf_tables}.

\section{Replication of \citet{williams+al.tacl21}}

We used the following UD treebanks:
\begin{itemize}[noitemsep]
    \item \textbf{Arabic}: PADT \citep{praguearabic,taji-etal-2017-universal};
    \item \textbf{Greek}: GDT \citep{Prokopidis05theoreticaland,prokopidis-papageorgiou-2017-universal};
    \item \textbf{Italian}: ISDT \citep{bosco-etal-2013-converting}, VIT \citep{tonelli-etal-2008-enriching};
    \item \textbf{Polish}: PDB \citep{wroblewska-2018-extended};
    \item \textbf{Portuguese}: GSD \citep{mcdonald-etal-2013-universal}, Bosque \citep{rademaker-etal-2017-universal};
    \item \textbf{Spanish}: AnCora \citep{taule-etal-2008-ancora}, GSD \citep{mcdonald-etal-2013-universal}.
\end{itemize}

\section{Additional Figures}

\begin{figure*}[h]
    \centering
    \includegraphics[width=\linewidth]{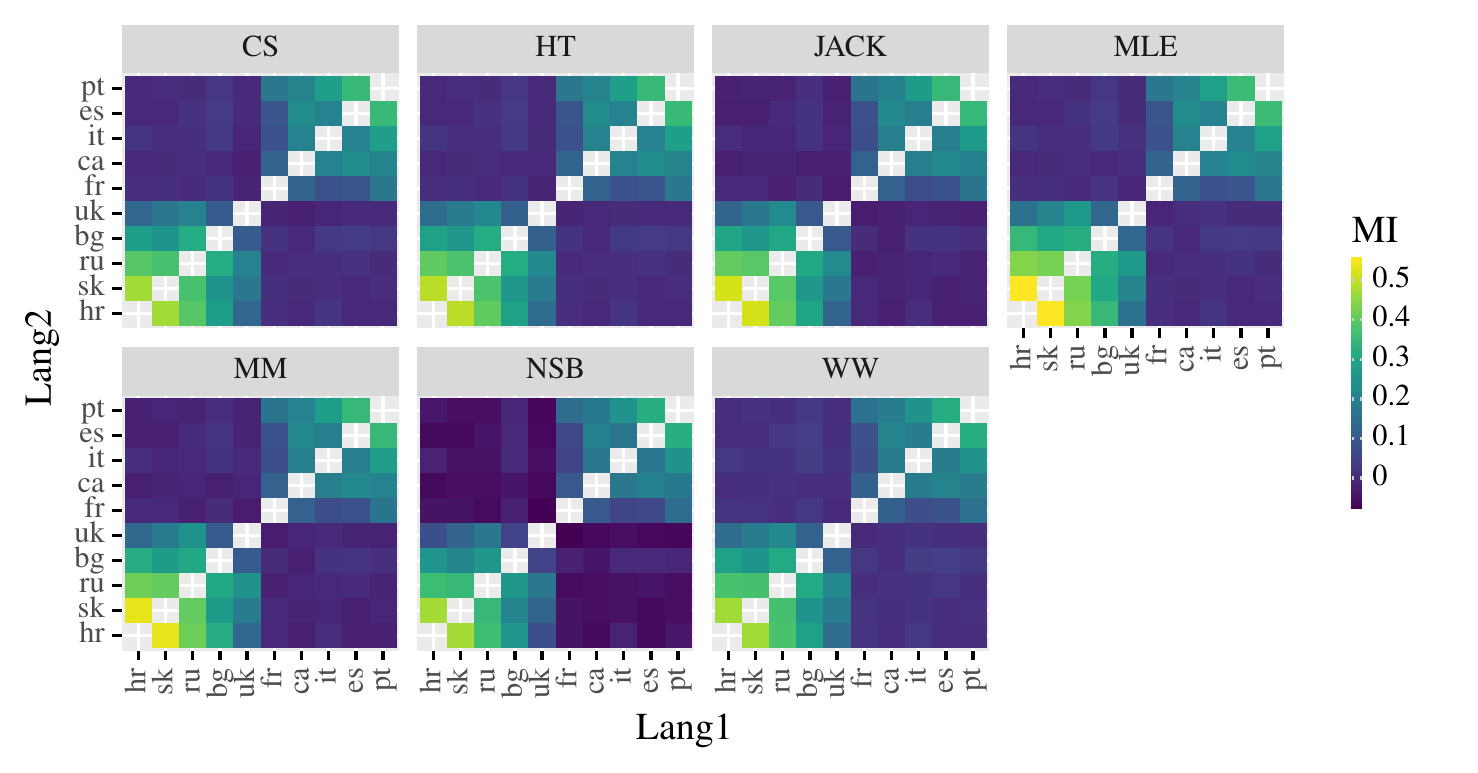}
    \caption{Mutual information between the gender partitions of language pairs with various estimators, replicating \citet{mccarthy-etal-2020-measuring}.}
    \label{fig:mccarthy}
\end{figure*}

\begin{figure*}[h]
    \centering
    \includegraphics[width=\linewidth]{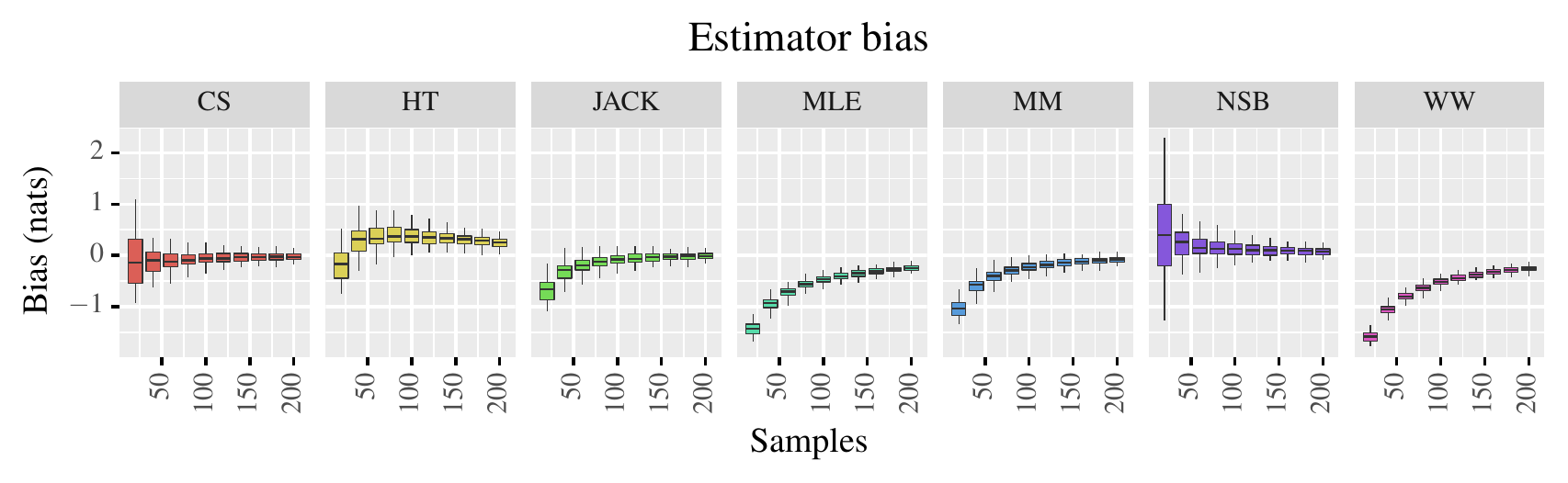}
    \caption{The distribution of bias for entropy over several estimators given variable sample size $N$, sampling from $100$ distributions taken from a symmetric Dirichlet prior with $K = 100$.}
    \label{fig:mle_bias}
\end{figure*}


\begin{figure*}[h]
    \centering
    \includegraphics[width=0.8\textwidth]{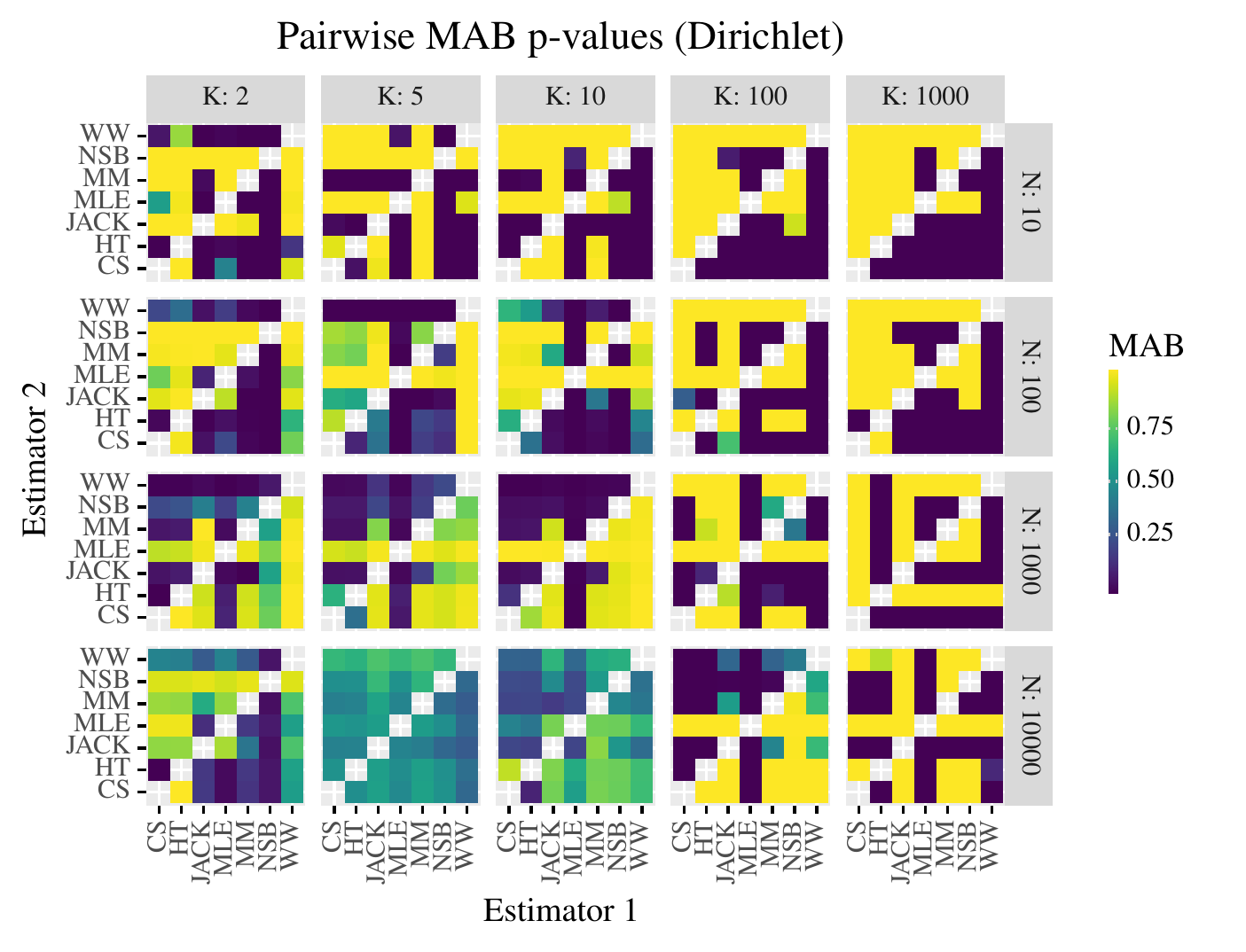}
    \includegraphics[width=0.8\textwidth]{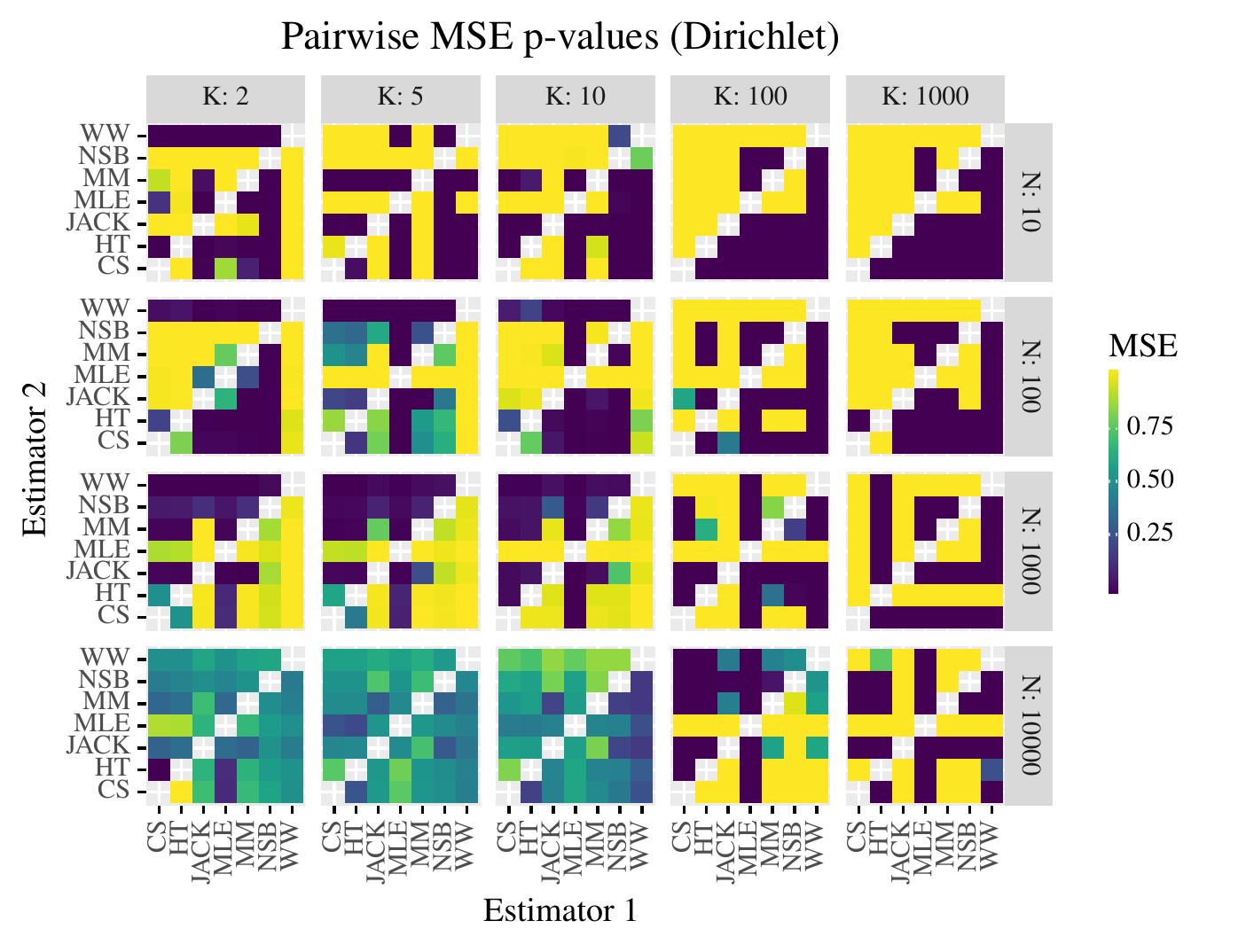}
    \caption{The heatmaps display the $p$-values calculated between pairs of estimators for mean absolute bias (MAB) and mean squared error (MSE) for Experiment 1. More purple values mean the estimator on the $y$-axis (Estimator 2) is better than the estimator on the $x$-axis (Estimator 1). Comparisons tend to become non-significant as $N$ increases, since all the estimators gradually converge to the true entropy.}
    \label{fig:estimators}
\end{figure*}


\newpage
\onecolumn

\section{Derivation of the Entropy Estimators}
Let $\calX = \{\xk\}_{k=1}^K$ be a finite set.
Let $p$ be a distribution over $\calX$.
The \defn{entropy} of $p$ is defined as\looseness=-1
\begin{equation}
    \ent(p) \defeq -\sum_{k=1}^K p_k \log p_k
\end{equation}
Given a dataset of $N$ samples $\calD$ sampled i.i.d.~from $p$, our goal is to estimate the entropy $\ent(p)$ from samples $\calD$ from the true distribution $p$. 
We will denote the count of an item $\xk$ as  $c(\xk) = \sum_{n=1}^N \mathbbm{1}\Big\{\xk = \widetilde{x}^{(n)}\Big\}$.
The \defn{maximum-likelihood estimate} (MLE) of $p$ given $\calD$ is denoted $\frac{\sum_{n=1}^N \mathbbm{1}\{\widetilde{x}^{(n)} = \xk\}}{N}$.
The \defn{plug-in estimate} of $\ent(p)$ is defined to be the estimate of $\ent(p)$ obtained by plugging the MLE estimate $\phat$ directly into the definition of entropy, i.e.,
\begin{equation}\label{eq:entropy-mle}
    \enthatmle = \ent(\phat) = -\sum_{k=1}^K \phatxk \log \phatxk = -\sum_{k=1}^K \frac{\cxk}{N} \log \frac{\cxk}{N}
\end{equation}
This section discusses the problems with \cref{eq:entropy-mle} as an estimator and provides detailed derivations of improved estimators found in the literature.

\subsection{The Plug-in Estimator is Negatively Biased}\label{app:mle_bias}
\begin{proposition}\label{prop:negative-bias}
The MLE entropy estimator in expectation underestimates true entropy, i.e., 
\begin{equation}
    \enthatmle = \mathbb{E}\left[\sum_{k=1}^K-\phatxk\log\phatxk\right] \leq \ent(p)
\end{equation}
\end{proposition}
\begin{proof}
The result is a simple consequence of Jensen's inequality and some basic manipulations:
\begin{align*}
 \mathbb{E}\left[\sum_{k=1}^K-\phatxk\log\phatxk\right] &=  \sum_{k=1}^K \mathbb{E}[-\phatxk\log\phatxk] & \text{(linearity of expectation)} \\
 &\leq- \sum_{k=1}^K\mathbb{E}[\phatxk]\log\mathbb{E}[\phatxk] & \text{(Jensen's inequality)}\\
    &= -\sum_{k=1}^K \pxk\log \pxk & \text{($\mathbb{E}[\phatxk]  = \pxk$)}\\
    &= \ent(p) & \text{(definition of entropy)}
\end{align*}
This completes the result.
\end{proof}

\subsection{Miller--Madow}

\begin{proposition}\label{prop:miller-madow}
Let $p$ be a categorical distribution over $\calX = \{x_1, \ldots, x_K\}$, i.e., a categorical distribution with support $K$.
\defndata{} Finally, let $\phat$ be the maximum-likelihood estimate computed on $\calD$. 
Then, we have
\begin{align}
  \bias{\enthatmle} &\defeq \expectp{\enthatmle} - \ent(p) \\
   &=-\frac{K-1}{2N} + o\left(N^{-1}\right)
\end{align}\
\end{proposition}
\begin{proof}
We start by taking a first-order Taylor expansion and take an expectation of both sides.
\begin{align}
    \enthatmle &= \underbrace{\ent(\phat, p)}_{\text{cross-entropy}} - \KL(\phat \mid\mid p) & \text{(\cref{lemma:taylor})} \\
    \expectp{\enthatmle} &= \expectp{\ent(\phat, p)} - \expectp{\KL(\phat \mid\mid p)} & \text{(expectation)} \\
    &= \expectp{-\sum_{k=1}^{K}{\phatxk\log{\pxk}}} - \expectp{\KL(\phat \mid\mid p)} & \text{(defn.~$\ent(p,q)$)} \\
    &= -\sum_{k=1}^{K}{\expectp{\phatxk\log{\pxk}}} - \expectp{\KL(\phat \mid\mid p)} & \text{(linearity)} \\ 
    &= -\sum_{k=1}^{K}{\expectp{\phatxk}\log{\pxk}} - \expectp{\KL(\phat \mid\mid p)} & \text{(algebra)} \\ 
    &= -\sum_{k=1}^{K}{\pxk\log{\pxk}} - \expectp{\KL(\phat \mid\mid p)} & %
    \text{(unbiased)} \\
    &= \ent(p)- \expectp{\KL(\phat \mid\mid p)} & \text{(defn.~of $\ent(p)$)}\\
\end{align}
This gives us:
\begin{align}
    \expectp{\enthatmle} - \ent(p) &= - \expectp{\KL(\phat \mid\mid p)} & \text{(subtract $\ent(p)$)}
\end{align}
Thus, we may compactly write the bias as:
\begin{align}
    \bias{\enthatmle} &=  \expectp{\ent(\phat)} - \ent(p) & \text{(definition of bias)} \\
    &= - \expectp{\KL(\phat \mid\mid p)} & \text{(above computation)}\\
     & \leq 0 & \text{(non-negativity of $\KL$)}
\end{align}
Now, we find a simpler expression for the remainder $\expectp{\KL(\phat \mid\mid p)}$.
Again, we start with a second-order Taylor expansion
\begin{align}
     \KL(p \mid\mid q) &=\sum_{x \in \calX}\frac{\Delta(x)^2}{2q(x)} + o\left(\Delta(x)^2\right)  & \text{(\cref{lemma:kl-taylor})}
\end{align}
around the point $\Delta(x) = p(x) - q(x)$.
Define $\phatxk = \frac{\cxk}{N}$ where $\cxk$ is the count of $\xk$ in the training set. We now simplify the first term:
\begin{align}
\expectp{\sum_{k=1}^K \frac{\Delta(\xk)^2}{2q(\xk)}}  &=\expectp{\sum_{k=1}^K \frac{(\phatxk - \pxk)^2}{2\pxk}} & \text{(definition of $\Delta(\xk)$)} \\
&= \expectp{\sum_{k=1}^K \frac{(\frac{\cxk}{N} - \pxk)^2}{2\pxk}} & \text{(definition of MLE)}\\
&= \expectp{\sum_{k=1}^K \frac{(\cxk- N \pxk)^2}{2N^2 \pxk}} & \text{($\times \sfrac{N}{N}$)} \\
&= \frac{1}{2N}\expectp{\sum_{k=1}^K \frac{(\cxk- N \pxk)^2}{N \pxk}} & \text{(pulling out $\sfrac{1}{2N}$)}\\
&= \frac{1}{2N}\expectp{\sum_{k=1}^K \frac{\splitfrac{\cxk^2 - 2\cxk N \pxk}{ + N^2 \pxk^2}}{N \pxk}} & \text{(exp.~the binomial)}\\
&= \frac{1}{2N}\sum_{k=1}^K \frac{\splitfrac{\expectp{\cxk^2} - 2N\pxk\expectp{\cxk}}{ + N^2 \pxk^2}}{N \pxk} & \text{(lin.~of expect.)} \\
&= \frac{1}{2N}\sum_{k=1}^K \frac{\splitfrac{N p_k (1-\pxk) + N^2\pxk^2}{ - 2N^2\pxk^2 + N^2 \pxk^2}}{N \pxk} & \text{(moments of MLE)} \\
\begin{split}
&= \frac{1}{2N}\sum_{k=1}^K \frac{N p_k (1-\pxk)}{N \pxk} \\&\qquad+ \underbrace{\frac{1}{2N}\sum_{k=1}^K \frac{ N^2\pxk^2 - 2N^2\pxk^2 + N^2 \pxk^2}{N \pxk}}_{=0}
\end{split}
\\
&= \frac{1}{2N}\sum_{k=1}^K \frac{\cancel{N \pxk} (1-\pxk)}{\cancel{N \pxk}} & \\
&= \frac{1}{2N}\sum_{k=1}^K (1-\pxk) & \text{(algebra)} \\
&= \frac{1}{2N}\underbrace{\sum_{k=1}^K 1}_{=K} - \frac{1}{2N}\underbrace{\sum_{k=1}^K \pxk}_{=1}& \text{(algebra)} \\
&= \frac{K-1}{2N}
\end{align}
Next, we simplify the second term, $o\left(\Delta(x)^2\right)$, in the MLE case:
\begin{align}
    \expectp{o\left(\Delta(x)^2\right)} &= \expectp{o\left((\phatxk - \pxk)^2\right)} & \text{(definition of $\Delta$)} \\
    &= \expectp{o\left(\left(\frac{\cxk}{N} - \pxk\right)^2\right)} & \text{(definition of MLE)}  \\
    &= \expectp{o\left(\frac{(\cxk - N \pxk)^2}{N^2}\right)} & \text{($\times \sfrac{N}{N}$)} \\
    &= \expectp{o\left(\frac{\cxk^2 - 2 \cxk N \pxk + N^2 \pxk^2}{N^2}\right)}\\
    &= o\left(\frac{\expectp{\cxk^2 - 2 \cxk N \pxk + N^2 \pxk^2}}{N^2}\right) & \text{(push exp.~through)}  \\
    &= o\left(\frac{\splitfrac{Np_k(1-\pxk) + N^2\pxk^2 }{- 2 N ^2\pxk^2 + N^2 \pxk^2}}{N^2}\right) \\
    &=o\left(\frac{N\pxk(1-\pxk)}{N^2}\right) & \text{(cancel terms)} \\
    &=o\left(\frac{\pxk(1-\pxk)}{N}\right) & \text{(cancel $N$ in fraction)} \\
    &=o\left(N^{-1}\right) & \text{(ignore constants)}
    \end{align}
Putting it all together, we get that $\bias{\ent(\phat)} = -\frac{K-1}{2N} + o\left(N^{-1}\right)$
which is the desired result.
\end{proof}
Interestingly, it can be seen that the negative bias of the MLE gets worse as the number of classes $K$ grows. 
Distributions with large $K$ pop up frequently when dealing with natural language. 
\begin{corollary}
The plug-in estimator of entropy is consistent.
\end{corollary}
\begin{proof}
From \cref{prop:miller-madow}, we have $\bias{\ent(\phat)} =-\frac{K-1}{2N} + o\left(N^{-1}\right)$. 
Clearly, as $N \rightarrow 0$, we have $\bias{\ent(\phat)} \rightarrow 0$, so the estimator is consistent.
One could also prove consistency through a simple application of the continuous mapping theorem.
\end{proof}

\begin{estimator}[Miller--Madow]
Let $p$ be a categorical over $K$ categories.
We seek to estimate the entropy $\ent(p)$.
\defndata{}
Then, the Miller--Madow estimator of $\ent(p)$ is given by\looseness=-1
\begin{equation}
   \enthat{mm}(\calD) \defeq \enthatmle + \frac{K-1}{2N}
\end{equation}
The Miller--Madow estimator is biased, however it is consistent.
\end{estimator}

\begin{lemma}\label{lemma:taylor}
The the first-order Taylor approximation of $\enthatmle$ around the distribution $p$ is given by\looseness=-1
\begin{equation}
\enthatmle = \ent(\phat, p) + R(p, \phat)
\end{equation}
where the remainder $R$ is given by
\begin{equation}
     R(p, \phat) = -\KL(\phat \mid\mid p)
\end{equation}
\end{lemma}
\begin{proof}
The result follows from direct computation. We start by taking the Taylor expansion of $\ent(\phat)$ around $\ent(p)$:
\begin{equation}
     \enthatmle = \ent(p) + \sum_{k=1}^K \frac{\partial}{\partial \pxk} \Big[ \ent(p)\Big] \Big(\phatxk - \pxk\Big) + \underbrace{R(p, \phat)}_{\text{remainder}} 
\end{equation}
   Our first order term can then be rewritten as follows:
    \begin{align}
    \sum_{k=1}^K \frac{\partial}{\partial \pxk} &\Big[ \ent(p)\Big] \Big(\phatxk - \pxk\Big)\\
    &=   \sum_{k=1}^K \frac{\partial}{\partial \pxk} \left[\sum_{k'=1}^K- p(x_{k'})\log p(x_{k'}) \right] \Big(\phatxk - \pxk\Big)   \\
        &=   \sum_{k=1}^K  \left[\sum_{k'=1}^K- \frac{\partial}{\partial \pxk} p(x_{k'})\log p(x_{k'}) \right] \Big(\phatxk - \pxk\Big)   & \text{(linearity)} \\
            &=   \sum_{k=1}^K  \left[\sum_{k'=1}^K \frac{\partial}{\partial \pxk} p(x_{k'})\log p(x_{k'}) \right] \Big(\pxk - \phatxk \Big)   & \text{(sign)} \\
        &=   \sum_{k=1}^K   \Big(1 + \log \pxk\Big)\Big(\pxk - \phatxk \Big)   &  \\
       &=   \sum_{k=1}^K \Big(\pxk - \phatxk \Big) + \log \pxk \left(\pxk - \phatxk \right)  \\
        &=   \sum_{k=1}^K \Big(\pxk - \phatxk \Big) + \sum_{k=1}^K \log \pxk \left(\pxk - \phatxk \right)    \\
        &=   \underbrace{\sum_{k=1}^K \pxk}_{=1} - \underbrace{\sum_{k=1}^K \phatxk}_{=1} + \sum_{k=1}^K \log \pxk \left(\pxk - \phatxk \right)   & \text{(distrib.~sum)}\\
        &=   \sum_{k=1}^K  \log \pxk \left(\pxk - \phatxk \right)   & \text{(simplify)} \\
        &=   \underbrace{\sum_{k=1}^K  \log \pxk \pxk}_{-\ent(p)} - \underbrace{\sum_{k=1}^K  \log \pxk \phatxk}_{\ent(p, \phat)}    & \text{(distrib.~sum)} \\
        &=  \ent(p, \phat) -\ent(p)  
\end{align}
Plugging this back into our Taylor expansion, we get the following:
\begin{equation}
        \enthatmle = \cancel{\ent(p)} - \cancel{\ent(p)} + \ent(p, \phat) + R(p, \phat) 
\end{equation}
Now, we see that this implies
\begin{align}
    R(p, \phat) &= \enthatmle - \ent(\phat, p) & \text{(algebra)} \\
    &= -\sum_{k=1}^K{\phatxk \log{\phatxk}} + \sum_{k=1}^K{\phatxk \log{\pxk}} & \text{(defn.)} \\
    &= -\sum_{k=1}^K{\left(\phatxk \log{\phatxk} - \phatxk \log{\pxk}\right)} & \text{(merge sums)} \\
    &= -\sum_{k=1}^K{\phatxk(\log{\phatxk} - \log{\pxk})} & \text{(factor out $\phatxk$)} \\
    &= -\sum_{k=1}^K{\phatxk\log{\frac{\phatxk}{\pxk}}} & \text{($\log$ algebra)} \\
    &= -\KL(\phat \mid\mid p) & \text{(defn.)}
\end{align}
which is the desired result.
\end{proof}

\begin{lemma}\label{lemma:kl-taylor}
Define $\Delta(x) = p(x) - q(x)$. 
The second-order Taylor expansion of $\KL(p \mid\mid q)$ around $\Delta(x)$ is given
by 
\begin{equation}
     \KL(p \mid\mid q) =\sum_{x \in \calX}\frac{\Delta(x)^2}{2q(x)} + o\left(\Delta(x)^2\right)
\end{equation}
\end{lemma}
\begin{proof}
Now we compute the series expansion of the KL-divergence. 
We first make a tricky substitution:
\begin{equation}\label{eq:tricky-sub}
\frac{p(x)}{q(x)} = \frac{q(x) + p(x) - q(x)}{q(x)} = 1 + \frac{p(x) - q(x)}{q(x)} = 1 + \frac{\Delta(x)}{q(x)}
\end{equation}
Now, we proceed with the derivation:
\begin{align}
     \KL(p& \mid\mid q) = \sum_{x \in \calX} p(x) \log\frac{p(x)}{q(x)} & \text{(defn.~of KL divergence)}\\
     &=  \sum_{x \in \calX} \left(q(x) + \Delta(x)\right)\log\left(1 + \frac{\Delta(x)}{q(x)}\right) & \text{(\cref{eq:tricky-sub})}\\
    &=  \sum_{x \in \calX} \left(q(x) + \Delta(x)\right)\left(\frac{\Delta(x)}{q(x)} - \frac{\Delta(x)^2}{2 q(x)^2} + o\left(\Delta(x)^2\right) \right) & \text{(Taylor expansion)} \\
     &=  \sum_{x \in \calX} \Delta(x) - \frac{\Delta(x)^2}{2 q(x)} + \frac{\Delta(x)^2}{q(x)} - \frac{\Delta(x)^3}{2 q(x)^2} + o\left(\Delta(x)^2\right) & \text{(distribute)} \\
          &=  \sum_{x \in \calX} \Delta(x) - \frac{\Delta(x)^2}{2 q(x)} + \frac{\Delta(x)^2}{q(x)} + o\left(\Delta(x)^2\right) & \text{(defn.~of $o$)} \\
         &=  \sum_{x \in \calX} \Delta(x) + \frac{\Delta(x)^2}{2q(x)} + o\left(\Delta(x)^2\right) & \text{(algebra)}\\
            &=  \underbrace{\sum_{x \in \calX} \Delta(x)}_{=0} + \sum_{x \in \calX}\frac{\Delta(x)^2}{2q(x)} + o\left(\Delta(x)^2\right) & \text{(split sums)} \\
            &= \sum_{x \in \calX}\frac{\Delta(x)^2}{2q(x)} + o\left(\Delta(x)^2\right)
\end{align}
which is the desired result.
\end{proof}

\subsection{Jackknife}
The jackknife resampling method is used to estimate the bias of an estimator and correct for it, by sampling all subsamples of size $N-1$ from the available sample of size $N$, computing their average for the statistic being estimated.

Generally, this reduces the order of the bias of an estimator from $O(N^{-1})$ to at most $O(N^{-2})$ \citep{friedl2002jackknife}.

\begin{estimator}[Jackknife]
Let $p$ be a categorical over $K$ categories.
We seek to estimate the entropy $\ent(p)$.
\defndata{}
Let ${\enthat{}^{\setminus n}(\calD)}$ be an estimate of the entropy from a sample with the $n$\textsuperscript{th} observation held out. Then, the \textbf{Jackknife estimator} is given by
\begin{equation}
    \enthat{jack}(\calD) \defeq N\,\enthat{mle}(\calD) - \frac{N-1}{N}\sum_{n=1}^N{ \enthat{mle}^{\setminus n}(\calD)}
\end{equation}
This estimator is derived from the jackknife-resampled estimate of the bias of the MLE estimator, multiplied by $N-1$.
\begin{equation}
    \enthat{jack}(\calD) - \enthat{mle}(\calD) = (N-1)\left(\enthatmle - \frac{1}{N}\sum_{n=1}^N{ \enthat{mle}^{\setminus n}(\calD)}\right)
\end{equation}
\end{estimator}


\subsection{Horvitz--Thompson}
\citet[HT;][]{horvitz-thompson} is a common estimator given a finite universe, which is our case as $K$ is finite. 
We omit a derivation a full here as it is well documented in other places \cite{vieira}.
However, we note that, in contrast to many applications of HT, the application of HT to entropy estimation results in a biased estimator as the function whose mean we seek to estimate is $\log p(\xk)$, which is dependent on the unknown distribution $p$. 



\begin{estimator}[Horvitz--Thompson]
Let $p$ be a categorical over $K$ categories.
We seek to estimate the entropy $\ent(p)$.
\defndata{}
Then the \defn{Horvitz--Thompson estimator} is defined as\looseness=-1
\begin{equation}
\enthat{ht}(\calD) \defeq -\sum_{k=1}^K \frac{\phatxk \log{\phatxk}}{1 - (1 - \phatxk)^N}
\end{equation}
where $1 - (1 - \phatxk)^N$ is an estimate of the \defn{inclusion probability}, i.e., the probability that $\xk$ appears in a random sample $\calD$ of size $N$.
\end{estimator}
We do not know of a simple expression for the bias of the Horvitz--Thompson entropy estimator, but one observation is that $\expectp{(1 - \phatxk)^N} > \expectp{(1-\pxk)^N}$ when $N > 1$ (justified by Jensen's inequality, since $x^N$, $N > 1$ is convex over $[0, 1]$); this is an overestimate of the true inclusion probability.


\subsection{Chao--Shen}
The Chao--Shen estimator builds upon Horvitz--Thompson by noting that that estimator does not correct for underestimation of number of classes $K$ and resulting effect on estimates of $\pxk$; i.e.~$1 - (1 - \phatxk)^N$ is always $0$ for a class not included in the sample even if the class is present in the true distribution. 
We can reweight the sample probabilities to compensate for missing classes using the notion of sample coverage.

\begin{definition}[Sample coverage]
We define the \defn{sample coverage} as
\begin{equation}\label{eq:sample-coverage}
    C = \sum_{k=1}^K \pxk \mathds{1}\Big\{x_k \in \calD\Big\}
\end{equation}
Definitionally, $(1-C)$ is then the probability of sampling an $x_k$ \emph{not} observed in the sample $\widetilde{\mathcal{X}}$. 
\end{definition}
However, exact computation of \cref{eq:sample_coverage} is impossible as we do not know the true distribution $p$.
Thus, \citet{chao2003nonparametric} fall back on a well-known estimator of $C$  that uses a technique from \citeauthor{good1953}--Turing (\citeyear{good1953}) smoothing.
Let $f_1$ be the number of classes with only one observation in the current sample, i.e, the number of singletons, then we can estimate the sample coverage as\looseness=-1
\begin{equation}\label{eq:sample_coverage}
    \widehat{C} \defeq 1 - \frac{f_1}{N}
\end{equation}
The Chao--Shen estimator, described below, simply re-scales the MLE estimate of probability $\phatxk$ in the HT estimator by $\widehat{C}$.
This corrects for the observed \emph{under}estimation of $p$'s entropy by HT.\looseness=-1
\begin{estimator}[Chao--Shen]
Let $p$ be a categorical over $K$ categories.
We seek to estimate the entropy $\ent(p)$.
\defndata{}
Let $\widehat{C}$, an estimate of sample coverage, be defined as in \cref{eq:sample_coverage}.
The \textbf{Chao--Shen estimator} is then defined as
\begin{equation}
\enthat{cs}(\calD) \defeq -\sum_{k=1}^K \frac{\widehat{C} \cdot \phatxk \log{(\widehat{C} \cdot \phatxk)}}{1 - (1 - \widehat{C} \cdot \phatxk)^N}
\end{equation}
\end{estimator}

\subsection{Wolpert--Wolf}

\begin{fact}[Derivative of an exponent]\label{fact:calc-exp}
\begin{equation}
   \frac{\mathrm{d}}{\mathrm{d} a} x^a = x^a \log x
\end{equation}
\end{fact}

\begin{fact}[Normalizer of a Dirichlet]\label{fact:dirichlet-normalizer}
The normalizer of a Dirichlet distribution is
\begin{equation}
    \int \delta\left(\sum_{k=1}^K x_k - 1\right) \prod_{k=1}^K x^{\alpha_k} \,\mathrm{d}{\boldsymbol{x}} = \frac{\prod_{k=1}^K \Gamma(\alpha_k)}{\Gamma\left(\sum_{k=1}^K\alpha_k \right)}
\end{equation}
A relatively easy proof of this fact makes use of a Laplace transform.
\end{fact}

\newcommand{\dirichletnormalizer}{\frac{ \Gamma\left(A\right)}{\prod_{k=1} \Gamma(\alpha_k)}}
\begin{estimator}[Wolpert--Wolf]
Let $p$ be a categorical over $K$ categories.
We seek to estimate the entropy $\ent(p)$.
\defndata{}
Then, the \textbf{Wolpert--Wolf estimator}
is given by
\begin{equation}
   \enthat{ww}(\calD \mid \valpha) \defeq  \digamma\left(\widetilde{A} + 1\right) - \sum_{k=1}^K\frac{\widetilde{\alpha}_k}{\widetilde{A}} \digamma(\widetilde{\alpha}_k + 1) 
\end{equation}
where $\cxk \defeq \sum_{n=1}^N \mathbbm{1}\{\widetilde{x}_n = x_k\}$, and we additionally define $\widetilde{\alpha}_k \defeq \cxk + \alpha_k$
and $\widetilde{A} \defeq \sum_{k=1}^K \widetilde{\alpha}_k$.
\end{estimator}

\begin{proposition}[Wolpert--Wolf]\label{prop:ww}
The expectation of entropy under a Dirichlet posterior $\Dirichlet(\valpha)$ where parameter $\valpha$ is given by
\begin{align}
    \expect{\ent(p) \mid \valpha} &\defeq \int \ent(p)\, \delta\left(\sum_{k=1}^K \pxk - 1\right) \dirichletnormalizer \prod_{k=1}^K \pxk^{\alpha_k-1} \ddp \\
    &=  \digamma\left(A + 1\right)- \sum_{k=1}^K\frac{\alpha_k}{A}\digamma(\alpha_k + 1) 
\end{align}
where $A \defeq \sum_{k=1}^K \alpha_k$.

\end{proposition}
\begin{proof}
Let $\Dirichlet(\alpha_1, \ldots, \alpha_K)$ be a Dirichlet posterior. 
The result follows by a series of manipulations:
\begin{align}
&\expect{\ent(p) \mid \valpha} = \int \ent(p)\, \delta\left(\sum_{k=1}^K \pxk - 1\right) \dirichletnormalizer  \prod_{k=1}^K \pxk^{\alpha_k-1} \ddp & \text{(defn.)} \\
&= \dirichletnormalizer \int \ent(p)\, \delta\left(\sum_{k=1}^K \pxk - 1\right)  \prod_{k=1}^K \pxk^{\alpha_k-1} \ddp & \\
&= \dirichletnormalizer \int \left(- \sum_{k=1}^K \pxk \log \pxk \right) \delta\left(\sum_{k=1}^K \pxk - 1\right) \prod_{k=1}^K p_k^{\alpha_k-1} \ddp & \text{(defn.~$\ent$)} \\
&= -\dirichletnormalizer  \sum_{k=1}^K \int \pxk \log \pxk \delta\left(\sum_{k=1}^K \pxk - 1\right) \prod_{k=1}^K \pxk^{\alpha_k-1} \ddp & \text{(linear.)}\\
&= - \dirichletnormalizer \sum_{k=1}^K \int \pxk^{\alpha_k} \log \pxk  \delta\left(\sum_{k=1}^K \pxk - 1\right) \prod_{\substack{j=1, \\ j \neq k}}^K p(x_j)^{\alpha_j-1} \ddp & \text{(algebra)}\\
&= - \dirichletnormalizer \sum_{k=1}^K \int \frac{\mathrm{d}}{\mathrm{d}\alpha_k}\pxk^{\alpha_k} \delta\left(\sum_{k=1}^K \pxk - 1\right) \prod_{\substack{j=1, \\ j \neq k}}^K p(x_j)^{\alpha_j-1} \ddp & \text{(fact \#1)} \\
&= -\dirichletnormalizer  \sum_{k=1}^K \int \frac{\mathrm{d}}{\mathrm{d}\alpha_k} \delta\left(\sum_{k=1}^K \pxk - 1\right) \pxk^{\alpha_k} \prod_{\substack{j=1, \\ j \neq k}}^K p(x_j)^{\alpha_j-1} \ddp & \text{(algebra)} \\ 
&= - \dirichletnormalizer \sum_{k=1}^K \frac{\mathrm{d}}{\mathrm{d}\alpha_k} \int \delta\left(\sum_{k=1}^K \pxk - 1\right) \pxk^{\alpha_k} \prod_{\substack{j=1, \\ j \neq k}}^K p(x_j)^{\alpha_j-1} \ddp &  \\
&= - \dirichletnormalizer \sum_{k=1}^K \frac{\mathrm{d}}{\mathrm{d}\alpha_k} \frac{\Gamma(\alpha_k + 1)\prod_{\substack{j=1, \\ j \neq k}}^K \Gamma(\alpha_j)}{\Gamma\left(\sum_{j=1}^K \alpha_j + 1\right)} & \text{(fact \#2)} \\
&= -\dirichletnormalizer  \sum_{k=1}^K \prod_{\substack{j=1, \\ j \neq k}}^K \Gamma(\alpha_j) \frac{\mathrm{d}}{\mathrm{d}\alpha_k} \frac{\Gamma(\alpha_k + 1)}{\Gamma\left(\sum_{j=1}^K \alpha_j + 1\right)} & \\
&= - \dirichletnormalizer \sum_{k=1}^K \prod_{\substack{j=1, \\ j \neq k}}^K \Gamma(\alpha_j) \frac{\digamma(\alpha_k + 1)\Gamma(\alpha_k + 1) \Gamma\left(\sum_{j=1}^K \alpha_j + 1\right)}{\Gamma\left(\sum_{j=1}^K \alpha_j + 1\right)^2} & \text{(derivative)}\\
&\qquad{} - \frac{\digamma(\sum_{j=1}^K \alpha_j + 1)\Gamma(\alpha_k + 1)\Gamma(\sum_{j=1}^K \alpha_k + 1)}{\Gamma\left(\sum_{j=1}^K \alpha_j + 1\right)^2} \nonumber\\
&= - \dirichletnormalizer \sum_{k=1}^K \prod_{\substack{j=1, \\ j \neq k}}^K \Gamma(\alpha_j) \frac{\splitfrac{\digamma(\alpha_k + 1)\Gamma(\alpha_k + 1) }{- \digamma(\sum_{j=1}^K \alpha_j + 1)\Gamma(\alpha_k + 1)}}{\Gamma\left(\sum_{j=1}^K \alpha_j + 1\right)} & \text{(simplify)}\\
&= -\dirichletnormalizer  \sum_{k=1}^K \prod_{\substack{j=1, \\ j \neq k}}^K \Gamma(\alpha_j) \frac{\splitfrac{\digamma(\alpha_k + 1)\Gamma(\alpha_k)\alpha_k}{ - \digamma(\sum_{j=1}^K \alpha_j + 1)\Gamma(\alpha_k)\alpha_k}}{\Gamma\left(\sum_{j=1}^K \alpha_j \right) A} & \text{(defn.~$\Gamma$)} \\
&= -\dirichletnormalizer  \frac{\prod_{k=1}^K \Gamma(\alpha_k)}{\Gamma\left(A\right)} \sum_{k=1}^K  \left(\frac{\alpha_k}{A}\digamma(\alpha_k + 1) - \frac{\alpha_k}{A}\digamma\left(\sum_{k=1}^K \alpha_k + 1\right)\right) & \text{(distrib.)} \\
&= -\sum_{k=1}^K  \left(\frac{\alpha_k}{A}\digamma(\alpha_k + 1) - \frac{\alpha_k}{A}\digamma\left(\sum_{k=1}^K \alpha_k + 1\right)\right) & \text{(cancel)} \\
&= -\sum_{k=1}^K  \left(\frac{\alpha_k}{A}\digamma(\alpha_k + 1) - \frac{\alpha_k}{A}\digamma\left(A + 1\right)\right) & \text{(defn.~$A$)} \\ 
&= -\sum_{k=1}^K  \frac{\alpha_k}{A}\digamma(\alpha_k + 1) + \sum_{k=1}^K  \frac{\alpha_k}{A}\digamma\left(A + 1\right) & \text{(distrib.)}  \\
&= -\sum_{k=1}^K  \frac{\alpha_k}{A}\digamma(\alpha_k + 1) + \digamma\left(A + 1\right) & \text{($\textstyle{\sum{a_k}\!=\!A}$)} \\
&= \digamma\left(A + 1\right) -\sum_{k=1}^K  \frac{\alpha_k}{A}\digamma(\alpha_k + 1)  & \text{(rearr.)}
\end{align}
which proves the result.
\end{proof}

\subsection{Nemenman--Shafee--Bialek}

\begin{estimator}[Nemenman--Shafee--Bialek]
Let $p$ be a categorical over $K$ categories.
We seek to estimate the entropy $\ent(p)$.
\defndata{}
Define the NSB density as\looseness=-1
\begin{equation}\label{eq:nsb-density}
    \pnsb(\alpha) \defeq \frac{K \trigamma\left(K \alpha + 1\right) - \trigamma(\alpha+1)}{\log K}
\end{equation}
where $\trigamma$ is the trigramma function.
Then, the \textbf{NSB estimator} is given by\looseness=-1
\begin{equation}\label{eq:appendix-nsb}
    \enthat{nsb}(\calD) \defeq \int_0^\infty \enthat{ww}(\calD \mid \alpha \cdot \boldsymbol{1})\,\pnsb(\alpha)\, \mathrm{d}\alpha
\end{equation}
The integral in \cref{eq:appendix-nsb} is typically computed by numerical integration.
\end{estimator}

To derive the Nemenman--Shafee--Bialek (NSB) estimator, we start with the idea that we would like a prior over distributions such that the distribution over expected entropy is uniform.
In other words, we are looking for a $\pnsb$ such that for $\alpha \sim \pnsb$, the values of $\expectp{\ent(p) \mid \alpha}$ are uniformly distributed over $[0, \log K]$.
This is a good idea since, a-priori, we do not know entropy of $p$ and, in the absence of any insight, we should assume the entropy could be anywhere in the range $[0, \log K]$.
We make the above intuition formal with the following proposition.

\begin{proposition}\label{prop:nsb}
Let $\pnsb$ be the NSB density given in \cref{eq:nsb-density}.
Then the following conditional expectation
\begin{align}
    \expectp{\ent(p) \mid \alpha} &\defeq \int \ent(p)\,\delta\left(\sum_{k=1}^K \pxk -1 \right)\,\frac{\Gamma\left(K \alpha \right)}{\Gamma(\alpha)^K}\prod_{k=1}^K \pxk^{\alpha - 1}\, \mathrm{d}p \\
    &= \digamma\left(K \alpha + 1\right)- \digamma(\alpha + 1)  & \text{(\cref{prop:ww})}
\end{align}
is uniformly distributed over $[0, \log K]$ when $\alpha \sim  \pnsb(\cdot)$, defined in \cref{eq:nsb-density}. 

\end{proposition}
\begin{proof}
First, we note that $\expectp{\ent(p) \mid \alpha}$ is a continuous, increasing function in $\alpha$.
We will not prove this formally, but it should make intuitive sense: $\alpha$ is a smoothing parameter and the more the distribution is smoothed, the more entropic it should be. 
From basic analysis, we know that a strictly continuous, increasing function has an inverse. 
The above means that we can view $\expectp{\ent(p) \mid \alpha}$ as a bijection from $\mathbb{R}_{\geq 0}$ to the interval $[0, \log K]$. 
 Our goal is to reparameterize the Uniform distribution in terms of $\alpha$.
 To that end, we define the function $g^{-1}(\alpha) \defeq \expectp{\ent(p) \mid \alpha} : \mathbb{R}_{\geq 0} \rightarrow [0, \log K]$ and 
perform a change-of-variables transform on \cref{eq:uniform} using $g^{-1}$.
We start with the continuous uniform over $[0, \log K]$, which is show below\looseness=-1
\begin{align}\label{eq:uniform}
p(H) & \defeq \underbrace{\frac{1}{\log K} \mathbbm{1}\Big\{ H \in [0, \log K] \Big\}}_{\text{uniform over $[0, \log K]$}} & \text{(defn.~of uniform dist)} 
\end{align}
Note $H$ is a random variable and unrelated to the functional $\ent(\cdot)$; the choice of letter intentionally reminds one that the variable represents the expected entropy of under a random distribution.
Now we apply the change-of-variables formula at $H = g^{-1}(\alpha)$ and manipulate:
\begin{align}
p(H) &= p(g^{-1}(\alpha)) \left|\frac{\mathrm{d}g^{-1}}{\mathrm{d}\alpha}(\alpha) \right| & \text{(change of variable)}  \\
&= \frac{1}{\log K}  \mathbbm{1}\Big\{ g^{-1}(\alpha) \in [0, \log K] \Big\} \left|\frac{\mathrm{d}g^{-1}}{\mathrm{d}\alpha}(\alpha) \right| & \text{(definition of $p$)} \\
 &=  \frac{1}{\log K}  \left|\frac{\mathrm{d}g^{-1}}{\mathrm{d}\alpha}(\alpha) \right|  & \text{(redundant indicator)}  \\
  &= \frac{1}{\log K}  \frac{\mathrm{d}g^{-1}}{\mathrm{d}\alpha}(\alpha) & \text{(derivative is positive)} \\
   &= \frac{K \trigamma\left(K \alpha + 1\right) - \trigamma(\alpha+1)}{\log K} & \text{(\cref{lemma:nsb-derivative})} \\
   &\defeq \pnsb(\alpha) & \text{(definition)}
\end{align}
 By construction, the prior $\pnsb(\alpha)$ has the property that the expected entropy $\expectp{\ent(p) \mid \alpha}$ where $\alpha \sim \pnsb(\cdot)$ is uniformly distributed over $[0, \log K]$, which we can see by reversing the above derivation.
 This proves the result.
\end{proof}
\citet{nsb} interpreted \cref{prop:nsb} in the following manner: As the variance of $\expectp{\ent(p) \mid \alpha}$, which is treated as a random variable since $\alpha$ is random, approaches 0, then the the NSB estimator implies a uniform prior over the entropy. 

\begin{lemma}[NSB Derivative]\label{lemma:nsb-derivative}
\end{lemma}
\begin{equation}
    \frac{\mathrm{d}}{\mathrm{d}\alpha} \left[ \digamma(K \alpha + 1) - \digamma(\alpha+1)  \right] =  K \trigamma(K \alpha+1) - \trigamma(\alpha+1)
\end{equation}
\begin{proof}
The proof follows by a straightforward computation:
\begin{align}
    \frac{\mathrm{d}}{\mathrm{d}\alpha} \left[ \digamma(K \alpha + 1) - \digamma(\alpha+1)  \right] &=  \frac{\mathrm{d}}{\mathrm{d}\alpha} \left[ \digamma(K \alpha + 1)\right]  - \frac{\mathrm{d}}{\mathrm{d}\alpha} \left[\digamma(\alpha+1) \right] & \text{(linearity)} \\
    &=  K \trigamma(K \alpha+1) - \trigamma(\alpha+1) & \text{(definition)}
\end{align}
where $\trigamma(x) \defeq \frac{\mathrm{d}}{\mathrm{d} x} \digamma(x)$. 
\end{proof}

\end{document}